\documentclass[11pt]{article}
\usepackage[T1]{fontenc}
\usepackage[margin=1in]{geometry}
\usepackage{booktabs}
\usepackage{longtable}
\usepackage{amsmath}
\usepackage{amssymb}
\usepackage{amsthm}
\usepackage{graphicx}
\usepackage{hyperref}

\newtheorem{proposition}{Proposition}

\title{ZIVARI-TLBO: A Zero-Cost Inter-Group Evaluated-Elite Relay Mechanism for Teaching--Learning-Based Optimization}
\author{Pezhman Zivari\\
Independent Researcher, Istanbul, T{\"u}rkiye\\
\href{mailto:pezhmanzivarizm@gmail.com}{\texttt{pezhmanzivarizm@gmail.com}}\\
ORCID: \href{https://orcid.org/0009-0001-4088-0589}{0009-0001-4088-0589}\\
Corresponding author: Pezhman Zivari}
\date{}

\hypersetup{
pdftitle={ZIVARI-TLBO: A Zero-Cost Inter-Group Evaluated-Elite Relay Mechanism for Teaching--Learning-Based Optimization},
pdfauthor={Pezhman Zivari},
pdfsubject={Teaching--Learning-Based Optimization; metaheuristic optimization; evaluated-elite relay; black-box optimization},
pdfkeywords={ZIVARI-TLBO, TLBO, teaching-learning-based optimization, metaheuristics, evolutionary computation, black-box optimization, evaluated-elite relay, zero-cost relay}
}

\begin{document}
\maketitle

\begin{abstract}
ZIVARI-TLBO is a grouped Teaching--Learning-Based Optimization (TLBO) method that augments an existing population-state controller with a fixed inter-group evaluated-elite relay. At each scheduled event, every group offers its already evaluated elite to the next group in a fixed ring; the elite replaces the receiver's worst eligible learner only when its stored objective value is better. Because the exact relay copies an already evaluated solution and its stored fitness, it requires no additional objective-function calls. The frozen \texttt{gts-v4-cm-fixed} implementation is evaluated under equal 10,000-evaluation budgets on eight classical functions at dimensions 10, 30, 50, and 100, with 30 matched seeds, and on five constrained engineering problems. A direct ablation against the same grouped landscape-aware controller without relay records 728/11/221 wins/ties/losses and a rank-biserial effect size of 0.624 across dimensions. In an eight-method multidimensional comparison, WOA obtains the best average rank (2.914) and ZIVARI-TLBO ranks second (3.382); ZIVARI-TLBO significantly outperforms TLBO, MCTLBO, DE, PSO, and GWO, loses significantly to WOA, and is not significantly different from HHO after Holm adjustment. Feasibility-aware engineering results are mixed and sensitive to the current static-penalty formulation. The evidence supports a scoped relay contribution and budget-consistent information-sharing mechanism, but not universal state-of-the-art, global-convergence, engineering-dominance, or CEC superiority claims.
\end{abstract}

\section{Introduction}
Teaching--Learning-Based Optimization (TLBO) is a population-based optimizer organized around teacher and learner phases and requires few algorithm-specific parameters \cite{rao2011tlbo,zou2019survey}. Grouped and multi-population extensions can preserve distinct search trajectories, but their effectiveness depends on how information moves between groups. Information exchange that is too weak can leave groups isolated, whereas aggressive transfer can reduce useful diversity. A second concern is experimental fairness: migration operators that generate and evaluate new candidates consume objective-function calls and must be counted within the same budget as the baseline methods.

This study investigates a deliberately conservative exchange mechanism. ZIVARI-TLBO retains the grouped search and population-state controller of the preceding development line, then adds a fixed ring relay that copies already evaluated group elites. The relay changes the receiving group's search state without generating a new objective evaluation. This distinction is important: all candidates generated by TLBO phases or an activated controller remain budget-counted; only the exact evaluated-elite copy is zero-cost in objective-function calls.

The contributions are:
\begin{enumerate}
  \item a fixed ring evaluated-elite relay that transfers a source-group elite to the worst eligible learner of the next group using stored objective values;
  \item an implementation-level budget audit demonstrating zero additional objective-function calls for the exact relay operation;
  \item a direct relay ablation against the same grouped landscape-aware controller without relay, using matched seeds and equal function-evaluation budgets;
  \item a multidimensional evaluation against TLBO, MCTLBO, DE, PSO, GWO, WOA, and HHO, plus a feasibility-aware constrained engineering audit;
  \item a statistically cautious analysis using paired outcomes, Wilcoxon signed-rank tests, Holm adjustment, rank-biserial effect sizes, and Friedman ranks where the comparison structure supports them; and
  \item a reproducibility package that preserves raw evidence, standardized data, plotting and table scripts, claim limitations, and directly inspected archives.
\end{enumerate}

\section{Related Work}
The original TLBO formulation demonstrated a compact teacher--learner search process for constrained mechanical design \cite{rao2011tlbo}, and a later survey documents the breadth of TLBO extensions and applications \cite{zou2019survey}. Grouped TLBO research has explored dynamic grouping, fuzzy grouping, multi-population cooperation, group-individual cooperation, reinforcement learning, and blended learning strategies \cite{zou2014dgstlbo,zhai2015fgtlbo,liao2023mctlbo,chen2023ctlbo,wu2022rltlbo,ma2024bltlbo}. More broadly, cooperative coevolution and island-model research show that decomposition, migration topology, and migration interval can materially alter population search behavior \cite{potter2000cc,skolicki2005migration}.

ZIVARI-TLBO is positioned narrowly within this literature. Its paper-facing contribution is not a claim that grouping, migration, or elite sharing is new in general. The contribution is the exact definition, budget audit, and controlled empirical evaluation of a fixed ring \emph{evaluated-elite} relay inside a frozen grouped TLBO framework. Unlike perturbative or newly generated transfer, the exact relay does not construct a new candidate and therefore does not require an objective-function call.

The multidimensional comparison includes the original or established population-based baselines DE, PSO, GWO, WOA, and HHO \cite{storn1997de,kennedy1995pso,mirjalili2014gwo,mirjalili2016woa,heidari2019hho}. It is intentionally not described as comprehensive modern state-of-the-art coverage: verified implementations of CMA-ES, SHADE, L-SHADE, and jSO are absent from the repository, although CMA-ES and success-history DE variants are important future comparators \cite{hansen2001cmaes,tanabe2013shade,tanabe2014lshade}. For constrained engineering, the study reports both the current static-penalty outcomes and a separate feasibility-first audit motivated by established constraint-handling principles \cite{deb2000constraints,coello2002constraints}.

\section{Proposed ZIVARI-TLBO}
\subsection{Method Overview}
ZIVARI-TLBO is the paper-facing name of the frozen \texttt{gts-v4-cm-fixed} implementation. The method partitions the population into groups and applies teacher and learner updates locally. It retains an inherited landscape-aware controller whose activation is inferred from population behavior, including improvement slope, diversity slope, group-best dispersion, and elite gap. No benchmark name is used by the activation logic. Any new candidates generated by an activated controller are evaluated and counted against the common function-evaluation budget.

The final method adds a fixed relay policy and an exact-copy relay mode. At every scheduled relay event, groups are paired in a ring, $g\rightarrow(g+1)\bmod G$. For each pair, the best already evaluated learner in the source group is compared with the worst eligible learner in the target group. The target excludes the current global-best position. If the stored source objective is better, both the source solution and its stored objective value replace the target; otherwise, no replacement occurs. Section~\ref{sec:formal-relay} formalizes this operation and its evaluation-budget interpretation.

\begin{quote}
ZIVARI-TLBO, standing for Zero-Cost Inter-Group Evaluated-Elite Relay for Teaching--Learning-Based Optimization, is implemented in the codebase as \verb|`gts-v4-cm-fixed`|.
\end{quote}

\subsection{Formal Definition and Evaluation-Cost Neutrality of the Relay Mechanism}
\label{sec:formal-relay}
Consider the continuous minimization problem
\begin{equation}
\min_{\mathbf{x}\in\Omega} f(\mathbf{x}),
\label{eq:problem}
\end{equation}
where $\mathbf{x}\in\mathbb{R}^{D}$ is a $D$-dimensional candidate solution, $\Omega$ is the feasible search domain, and $f(\cdot)$ is the objective function to be minimized. At iteration $t$, the population is
\begin{equation}
P^{(t)} =
\left\{
\mathbf{x}_1^{(t)},\mathbf{x}_2^{(t)},\ldots,\mathbf{x}_N^{(t)}
\right\},
\label{eq:population}
\end{equation}
where $N$ is the population size and each $\mathbf{x}_i^{(t)}$ is a candidate vector in the $D$-dimensional search space. The population is partitioned into $G$ mutually disjoint groups:
\begin{equation}
P^{(t)} = \bigcup_{g=1}^{G} P_g^{(t)},
\qquad
P_i^{(t)} \cap P_j^{(t)} = \emptyset
\quad \text{for } i\neq j.
\label{eq:partition}
\end{equation}
Each group acts as a local learning subpopulation. This structure is intended to let different groups search different regions while retaining local teacher--learner pressure. It can help maintain distinct search trajectories relative to a single global teacher and may reduce premature concentration within the adopted experimental protocol; it does not guarantee diversity, escape from local optima, or global convergence.

ZIVARI-TLBO can therefore be interpreted as a grouped TLBO framework. Within each $P_g^{(t)}$, teacher and learner phases are performed locally, and the best individual serves as the local teacher or elite. The group elite is defined by
\begin{equation}
\mathbf{e}_g^{(t)}
=
\arg\min_{\mathbf{x}\in P_g^{(t)}} f(\mathbf{x}).
\label{eq:group-elite}
\end{equation}
Importantly, $\mathbf{e}_g^{(t)}$ is not a newly generated relay candidate. It is an existing member of group $g$ whose objective value has already been computed before the relay operation.

The frozen implementation protects the current global-best position from being overwritten. Let $\widetilde{P}_b^{(t)}$ denote the eligible members of receiver group $b$ after excluding that protected position. For a minimization problem, the receiver target is
\begin{equation}
\mathbf{w}_b^{(t)}
=
\arg\max_{\mathbf{x}\in \widetilde{P}_b^{(t)}} f(\mathbf{x}).
\label{eq:receiver-worst}
\end{equation}
If the eligible set is empty, no relay replacement is attempted for that receiver. Otherwise, $\mathbf{w}_b^{(t)}$ is the worst eligible receiver member and may be replaced only when the incoming evaluated elite is better.

For a sender group $a$ and receiver group $b$, the relayed solution is defined as
\begin{equation}
\mathbf{r}_{a\rightarrow b}^{(t)}
=
\mathbf{e}_a^{(t)}.
\label{eq:relay}
\end{equation}
This relay is an exact evaluated-elite copy. It is not mutation, perturbation, interpolation, noisy transfer, recombination, or a soft update. The relayed vector $\mathbf{r}_{a\rightarrow b}^{(t)}$ is identical to the already evaluated elite $\mathbf{e}_a^{(t)}$. Therefore, the relay operation does not construct a new point in the search space; it transfers existing evaluated information between groups. The stored objective value is transferred together with the solution:
\begin{equation}
f\!\left(\mathbf{r}_{a\rightarrow b}^{(t)}\right)
=
f\!\left(\mathbf{e}_a^{(t)}\right).
\label{eq:stored-fitness}
\end{equation}
This equality is valid because $\mathbf{r}_{a\rightarrow b}^{(t)}$ is exactly the same evaluated vector as $\mathbf{e}_a^{(t)}$, not a newly generated candidate. The evaluated-elite relay operation does not require additional objective-function calls because it transfers an already evaluated elite solution between groups.

Using the stored values, the greedy replacement rule is
\begin{equation}
\begin{aligned}
&\text{if }\;
f\!\left(\mathbf{e}_a^{(t)}\right)
<
f\!\left(\mathbf{w}_b^{(t)}\right):
&&
\mathbf{w}_b^{(t)} \leftarrow \mathbf{e}_a^{(t)},
\quad
f\!\left(\mathbf{w}_b^{(t)}\right)
\leftarrow
f\!\left(\mathbf{e}_a^{(t)}\right);
\\
&\text{otherwise:}
&&
P_b^{(t)}\text{ is unchanged.}
\end{aligned}
\label{eq:greedy-replacement}
\end{equation}
This rule prevents the relay from injecting a worse solution into the receiver group under minimization. The comparison and accepted replacement use already available objective values and therefore do not call $f(\cdot)$.

The fixed schedule in the frozen implementation forms a cyclic ring. With groups indexed from $1$ to $G$, sender $a$ communicates with
\begin{equation}
b = (a \bmod G) + 1.
\label{eq:cyclic-relay}
\end{equation}
At each scheduled event, all nonempty ring pairs are considered sequentially in this fixed order; the archived default schedules a relay event after every group epoch. The source elite is read from the current group state when its pair is processed. Consequently, an exact evaluated elite accepted by one group may be relayed onward later in the same ring event, still with its valid stored objective value and without a new objective-function call. The schedule does not adapt sender or receiver selection to benchmark identity or population-state diagnostics.

In general migration or perturbative transfer schemes, the availability of a stored fitness value depends on whether the transferred entity is an already evaluated solution or a newly generated candidate. In ZIVARI-TLBO, this availability is guaranteed by design because the relay transfers the evaluated elite together with its stored objective value. General migration may transfer existing individuals, modified individuals, or newly generated candidates, depending on the design. If mutation, perturbation, interpolation, recombination, or noise modifies a transferred vector, its objective value is not available in advance and it must be evaluated. In contrast, ZIVARI-TLBO deliberately transfers the exact evaluated elite without modification, so its stored objective value remains valid. Table~\ref{tab:relay-vs-migration} summarizes this distinction without implying that all migration schemes require new evaluations.

\begin{table}[ht]
\centering
\footnotesize
\begin{tabular}{@{}p{0.30\linewidth}p{0.31\linewidth}p{0.29\linewidth}@{}}
\toprule
\textbf{Property} & \textbf{General migration or perturbative transfer} & \textbf{ZIVARI evaluated-elite relay} \\
\midrule
Inter-group information exchange & Yes & Yes \\
Uses elite information & Sometimes & Yes, group elite \\
May generate a new candidate & Yes, depending on design & No \\
May use perturbation/noise/interpolation & Yes, depending on design & No \\
Stored fitness value guaranteed & Not necessarily & Yes \\
Additional objective-function call required & Possibly & No \\
Zero-cost with respect to NFE & Only if an already evaluated solution is transferred without modification & Yes, by design \\
\bottomrule
\end{tabular}
\caption{Careful distinction between general migration or perturbative transfer and the evaluated-elite relay used by frozen ZIVARI-TLBO.}
\label{tab:relay-vs-migration}
\end{table}

\begin{proposition}[Objective-evaluation neutrality of the ZIVARI-TLBO relay]
\label{prop:evaluation-neutrality}
The evaluated-elite relay operation in ZIVARI-TLBO is objective-evaluation neutral.
\end{proposition}
\begin{proof}
Let $NFE_{\mathrm{before\ relay}}^{(t)}$ be the number of objective-function evaluations immediately before applying the relay at iteration $t$, and let $NFE_{\mathrm{after\ relay}}^{(t)}$ be the corresponding number immediately afterward. By Eqs.~\eqref{eq:relay} and \eqref{eq:stored-fitness}, the relay copies an already evaluated elite vector and its stored objective value. The comparison in Eq.~\eqref{eq:greedy-replacement} also uses stored values. Since no call to $f(\cdot)$ is made during the relay operation,
\begin{equation}
NFE_{\mathrm{after\ relay}}^{(t)}
=
NFE_{\mathrm{before\ relay}}^{(t)}.
\end{equation}
Therefore,
\begin{equation}
\Delta NFE_{\mathrm{relay}}
=
NFE_{\mathrm{after\ relay}}^{(t)}
-
NFE_{\mathrm{before\ relay}}^{(t)}
=
0.
\end{equation}
\end{proof}

If $m$ evaluated elite vectors of dimension $D$ are relayed, their bookkeeping and vector-copying cost is $O(mD)$. If at most one elite is relayed per group in an event, then $m\leq G$, and this cost is bounded by $O(GD)$. The number of additional objective-function evaluations remains zero. The relay operation may incur a small bookkeeping or vector-copying cost, but this cost is distinct from objective-function evaluation cost. In black-box optimization, the dominant budget is commonly measured by the number of objective-function evaluations; with respect to this budget, the evaluated-elite relay is neutral.

ZIVARI-TLBO deliberately uses an exact evaluated-elite relay rather than adaptive perturbation, noisy transfer, soft interpolation, or new-candidate generation. Perturbative or soft-transfer variants may be interesting for future work, but they generally create modified candidate vectors whose objective values are not available in advance. Such variants would require additional objective-function evaluations and would therefore fall outside the zero-cost relay mechanism studied in this paper. These variants are neither part of the present method nor validated by its experiments.

Because the relayed elite is already present elsewhere in the population, the copy cannot immediately improve the stored global-best objective. Its intended role is instead to alter the receiving group's composition so that subsequent local teacher and learner updates can exploit information discovered by another group. This analytical result proves evaluation-cost neutrality of the relay operation only; performance advantages must be interpreted empirically within the adopted benchmark suites and equal-budget experimental protocol.

\subsection{Pseudocode}
Algorithm~\ref{alg:zivari-tlbo} states the frozen workflow and distinguishes budget-counted candidate generation from the zero-call relay.

\begin{table}[ht]
\centering
\small
\begin{tabular}{p{0.96\linewidth}}
\toprule
\textbf{Algorithm 1: ZIVARI-TLBO with fixed ring evaluated-elite relay} \\
\midrule
\textbf{Input:} objective function $f$, bounds, population size $N$, groups $G$, maximum evaluation budget $B$, relay period $\tau$. \\
1. Initialize $N$ learners uniformly within the search bounds and evaluate each learner. \\
2. Partition the population into $G$ learner groups. \\
3. While the evaluation budget is not exhausted, repeat the following steps. \\
4. For each group, run the TLBO teacher phase locally using the group teacher and group mean. Evaluate only newly generated candidates and accept improvements according to the implemented replacement rule. \\
5. For each group, run the TLBO learner phase locally through pairwise learner interactions. Evaluate only newly generated candidates and accept improvements according to the implemented replacement rule. \\
6. At the end of the epoch, infer whether the inherited landscape-aware controller should activate from population-state signals; count every candidate evaluation produced by an activated controller against $B$. \\
7. Every $\tau$ epochs, form fixed ring pairs $a\rightarrow b$, where $b=(a\bmod G)+1$, and process them sequentially in fixed order. \\
8. For each pair, identify the already evaluated elite $\mathbf{e}_a^{(t)}$ from the current source-group state and use its stored value $f(\mathbf{e}_a^{(t)})$. Identify the worst eligible non-global-best receiver member $\mathbf{w}_b^{(t)}$ and compare their stored objective values. \\
9. If $f(\mathbf{e}_a^{(t)})<f(\mathbf{w}_b^{(t)})$, copy both $\mathbf{e}_a^{(t)}$ and its stored objective value into the receiver position; otherwise leave the receiver unchanged. Do not call $f$ during relay. \\
10. Update group bests and the global best solution using stored objective values. \\
11. Stop immediately when the evaluation budget $B$ is reached. \\
\textbf{Output:} the best evaluated solution and its objective value. \\
\bottomrule
\end{tabular}
\caption{Pseudocode aligned with the frozen ZIVARI-TLBO implementation. The code alias is \texttt{gts-v4-cm-fixed}; the zero-cost property applies to the exact relay operation, while all generated candidates remain budget-counted.}
\label{alg:zivari-tlbo}
\end{table}

\section{Complexity Analysis}
For population size $N$, dimension $D$, number of groups $G$, and $T$ iterations, local TLBO updates have arithmetic cost proportional to $O(TND)$ in addition to objective-function evaluations. Population-state diagnostics and controller operations add bookkeeping and, when they generate candidates, budget-counted evaluations. As established in Proposition~\ref{prop:evaluation-neutrality}, relaying $m$ evaluated elites costs $O(mD)$ for vector copying, is bounded by $O(GD)$ when $m\leq G$, and satisfies $\Delta NFE_{\mathrm{relay}}=0$. Table~\ref{tab:zivari-complexity-budget} summarizes this distinction.

\begin{table}[ht]
\centering
\small
\begin{tabular}{p{0.23\linewidth}p{0.70\linewidth}}
\toprule
Item & Paper-ready statement \\
\midrule
Base TLBO update & Teacher and learner phases require arithmetic-level update cost proportional to $O(TND)$, plus objective-function calls. \\
Grouping overhead & Group partitioning and group-level bookkeeping add lightweight overhead dependent on the number of groups. \\
Landscape-aware controller & Population-state diagnostics and any activated controller operations are inherited from the pre-v4 grouped framework; every objective evaluation generated by these operations is counted against the common budget. \\
Evaluated-elite relay & The evaluated-elite relay operation does not require additional objective-function calls because it transfers an already evaluated elite solution between groups. \\
Relay bookkeeping & Relaying $m$ evaluated $D$-dimensional elites requires $O(mD)$ vector-copying work; if $m\leq G$, this is bounded by $O(GD)$. \\
Relay evaluation delta & The exact relay uses stored objective values, so $\Delta NFE_{\mathrm{relay}}=0$. \\
Budget control & All methods are compared under the same function-evaluation budget. \\
Earlier ablations & Soft or noisy transfer variants consumed evaluations when they created new candidates. \\
Final method & ZIVARI-TLBO uses the fixed exact evaluated-elite relay mechanism. \\
\bottomrule
\end{tabular}
\caption{Complexity and evaluation-budget interpretation for ZIVARI-TLBO.}
\label{tab:zivari-complexity-budget}
\end{table}

Figure~\ref{fig:runtime-budget} reports descriptive runtime measurements for the expanded dimension-30 layer. Hardware context is unavailable, so the figure is not interpreted as a platform-independent complexity comparison.

\begin{figure}[ht]
\centering
\includegraphics[width=0.95\linewidth]{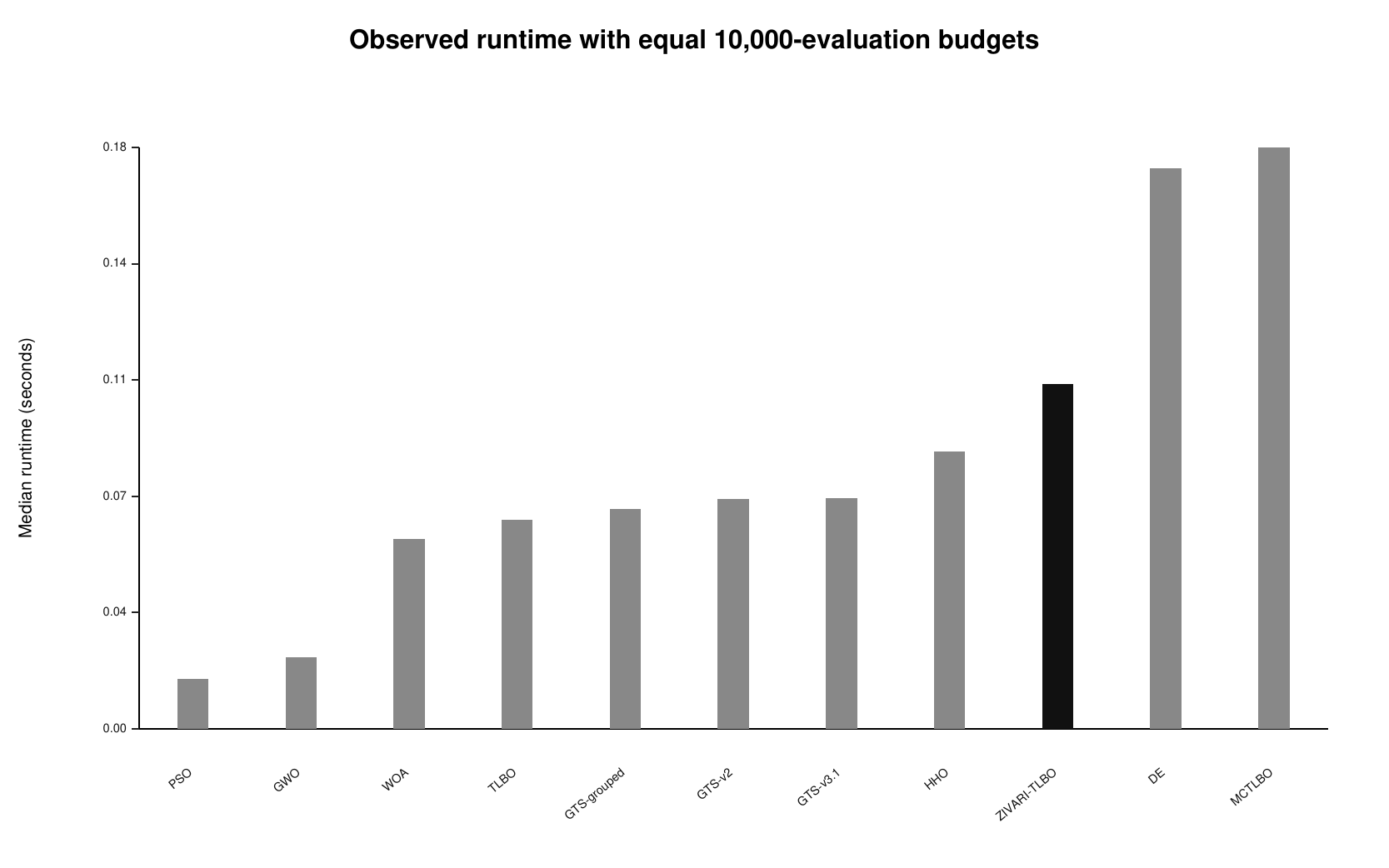}
\caption{Observed median runtime for the expanded dimension-30 classical layer under equal 10,000-evaluation budgets. Hardware context is unspecified; runtime is descriptive only.}
\label{fig:runtime-budget}
\end{figure}

\section{Experimental Setup}
All comparisons use an equal maximum objective-function evaluation budget. Unless otherwise stated, each run uses population size 60, 30 independent seeds, a 1,000-iteration safety limit, and a maximum of 10,000 objective evaluations; the evaluation budget is the operative stopping rule. Lower objective values and lower ranks are better.

For the frozen ZIVARI-TLBO configuration, the archived defaults use five groups, a five-iteration group epoch, fixed relay at every group epoch ($\tau=1$), exact evaluated-elite copy, and all fixed ring pairs. Remaining controller thresholds and variant defaults are preserved in \texttt{TI.py}.

The classical suite contains Ackley, Griewank, Levy, Rastrigin, Rosenbrock, Schwefel, Sphere, and Zakharov at dimensions 10, 30, 50, and 100. The direct relay ablation compares ZIVARI-TLBO with the existing \texttt{gts-v3.1-landscape-aware} condition, reported as ``ZIVARI-TLBO without relay.'' Programmatic default comparison confirms that the conditions share the grouped landscape-aware controller and all common settings; the evaluated-elite relay is the controlled difference. The broader multidimensional comparison includes ZIVARI-TLBO, TLBO, MCTLBO, DE, PSO, GWO, WOA, and HHO. All matched comparisons use the same functions, seeds, population size, evaluation budget, stopping rule, and boundary handling.

The engineering layer contains welded beam, pressure vessel, tension/compression spring, speed reducer, and three-bar truss design. It uses the current static quadratic penalty with scale $10^7$ and reports objective values, penalized ranks, strict-feasibility rates at tolerance $10^{-6}$, and constraint-violation summaries. A Deb-style feasibility-first ranking is added only as an external analysis layer; it does not modify optimizer behavior or candidate acceptance. Recorded timing hardware and operating-system details are unspecified. No verified CEC implementation, official data package, or validated wrapper is present, so CEC results are not included.

Table~\ref{tab:experimental-protocol} summarizes the principal comparison scopes. The same random-seed sequence, population size, and evaluation budget are used within each matched scope.

\begin{table}[ht]
\centering
\footnotesize
\begin{tabular}{@{}p{0.23\linewidth}p{0.20\linewidth}p{0.37\linewidth}p{0.10\linewidth}@{}}
\toprule
\textbf{Layer} & \textbf{Dimensions / problems} & \textbf{Compared conditions} & \textbf{Runs} \\
\midrule
Direct relay ablation & $D=10,30,50,100$; eight classical functions & ZIVARI-TLBO; same grouped landscape-aware controller without relay & 30 seeds \\
External classical & $D=10,30,50,100$; eight classical functions & ZIVARI-TLBO, TLBO, MCTLBO, DE, PSO, GWO, WOA, HHO & 30 seeds \\
Engineering & five constrained designs & Eleven implemented population-based methods; focal interpretation emphasizes ZIVARI-TLBO, MCTLBO, and DE & 30 seeds \\
\bottomrule
\end{tabular}
\caption{Audited experimental scopes. Every run uses population size 60 and a maximum of 10,000 objective evaluations; lower objective values and ranks are better.}
\label{tab:experimental-protocol}
\end{table}

\section{Classical Benchmark Results}
\subsection{Direct Ablation of the Evaluated-Elite Relay}
The direct ablation isolates the paper-facing contribution more closely than comparisons with different optimizers. ZIVARI-TLBO and the relay-disabled condition use the same grouped landscape-aware controller, matched seeds, boundary handling, stopping rule, and 10,000-evaluation budget. Table~\ref{tab:direct-relay-ablation} reports paired results from the perspective of ZIVARI-TLBO.

\begin{table}[ht]
\centering
\small
\begin{tabular}{lrrrrrr}
\toprule
Scope & W & T & L & Wilcoxon $p$ & Holm $p$ & RBC \\
\midrule
10 & 187 & 8 & 45 & 2.032e-14 & 6.097e-14 & 0.579 \\
30 & 196 & 0 & 44 & 1.580e-20 & 6.320e-20 & 0.692 \\
50 & 173 & 2 & 65 & 7.938e-14 & 3.175e-13 & 0.559 \\
100 & 172 & 1 & 67 & 3.323e-18 & 1.661e-17 & 0.649 \\
ALL DIMS & 728 & 11 & 221 & 3.528e-62 & 1.411e-61 & 0.624 \\
\bottomrule
\end{tabular}
\caption{Direct relay ablation under matched seeds and equal 10,000-evaluation budgets. W/T/L and rank-biserial correlation (RBC) are reported from the perspective of ZIVARI-TLBO against the same controller without relay.}
\label{tab:direct-relay-ablation}
\end{table}

Across all dimensions, ZIVARI-TLBO records 728/11/221 wins/ties/losses, with Holm-adjusted $p=1.411\times10^{-61}$ and rank-biserial correlation 0.624. Its two-condition average rank is 1.235, compared with 1.765 without relay. The advantage remains positive at every tested dimension. Figure~\ref{fig:direct-relay-ablation} shows the dimensional W/T/L pattern.

\begin{figure}[ht]
\centering
\includegraphics[width=0.95\linewidth]{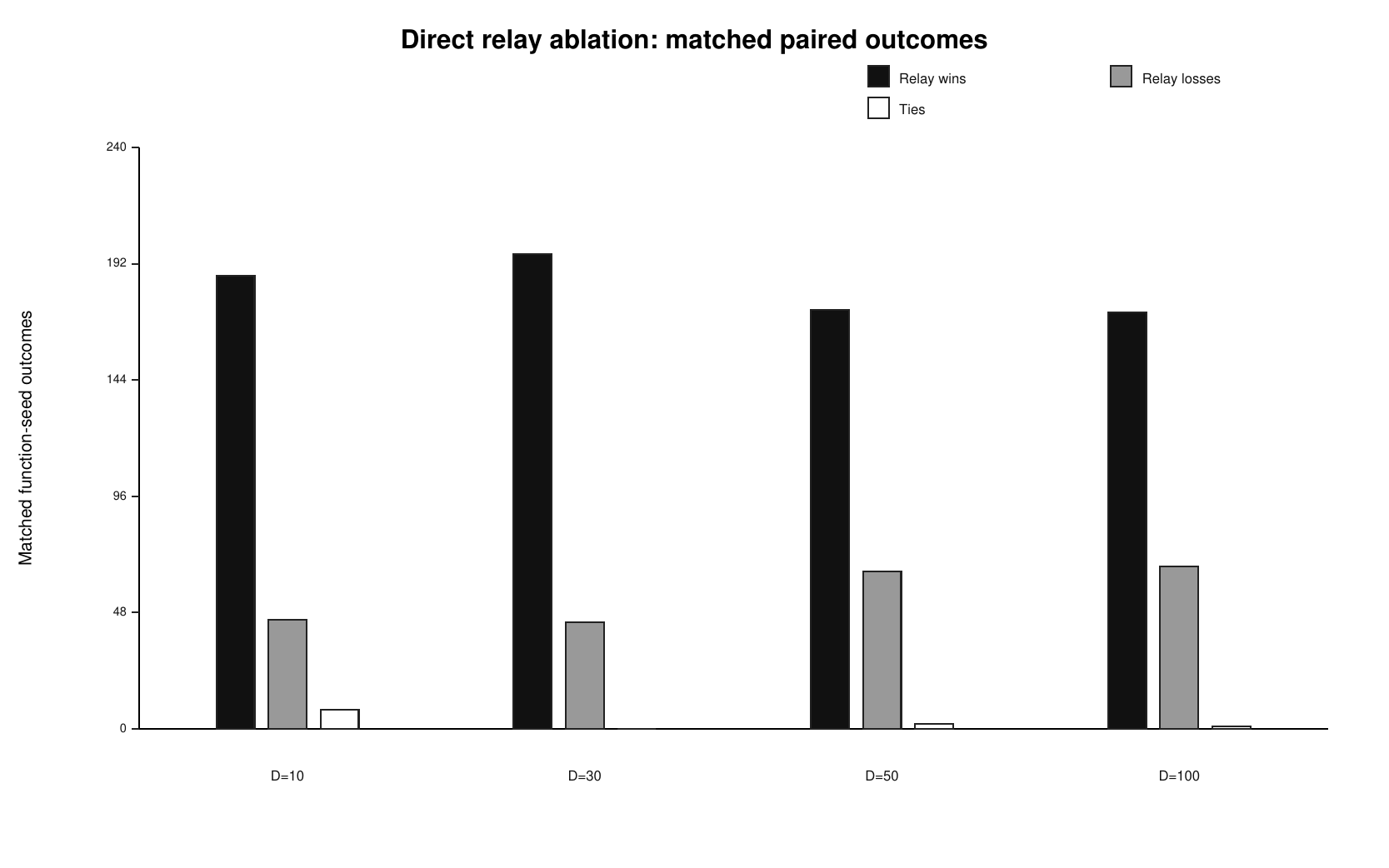}
\caption{Direct evaluated-elite relay ablation under matched seeds and equal 10,000-evaluation budgets. Wins, ties, and losses are counted from the perspective of ZIVARI-TLBO against the same controller without relay.}
\label{fig:direct-relay-ablation}
\end{figure}

This controlled result supports the relay's contribution within the tested implementation and protocol. It is not a universal causal proof: an exact copy interacts with the inherited controller, grouping, functions, dimensions, and fixed schedule, and the experiment does not establish that every relay policy or every landscape benefits from elite transfer.

\subsection{Multi-Dimensional External Baseline Evaluation}
Table~\ref{tab:external-multidim-ranks} reports average ranks for eight methods across the four tested dimensions, and Table~\ref{tab:external-multidim-pairwise} reports matched all-dimensions comparisons from the perspective of ZIVARI-TLBO.

\begin{table}[ht]
\centering
\small
\begin{tabular}{lrrrrr}
\toprule
Method & $D=10$ & $D=30$ & $D=50$ & $D=100$ & All dimensions \\
\midrule
WOA & 4.006 & 3.046 & 2.565 & 2.040 & 2.914 \\
ZIVARI-TLBO & 3.421 & 3.300 & 3.408 & 3.400 & 3.382 \\
HHO & 5.004 & 3.344 & 3.044 & 2.469 & 3.465 \\
MCTLBO & 3.962 & 3.983 & 3.908 & 3.850 & 3.926 \\
GWO & 4.152 & 3.677 & 4.121 & 5.000 & 4.237 \\
TLBO & 4.871 & 5.125 & 5.046 & 5.071 & 5.028 \\
PSO & 4.258 & 6.204 & 6.713 & 7.217 & 6.098 \\
DE & 6.325 & 7.321 & 7.196 & 6.954 & 6.949 \\
\bottomrule
\end{tabular}
\caption{Average ranks in the multi-dimensional external-baseline evaluation. Lower ranks are better; ranks summarize the tested functions, seeds, dimensions, and equal-budget protocol only.}
\label{tab:external-multidim-ranks}
\end{table}

\begin{table}[ht]
\centering
\small
\begin{tabular}{lrrrrr}
\toprule
Comparator & W & T & L & Holm $p$ & RBC \\
\midrule
TLBO & 796 & 0 & 164 & 1.175e-68 & 0.656 \\
MCTLBO & 731 & 0 & 229 & 2.521e-29 & 0.422 \\
DE & 929 & 0 & 31 & 6.914e-151 & 0.978 \\
PSO & 772 & 0 & 188 & 7.393e-51 & 0.563 \\
GWO & 577 & 0 & 383 & 3.064e-46 & 0.536 \\
WOA & 285 & 1 & 674 & 1.627e-07 & -0.200 \\
HHO & 342 & 1 & 617 & 1.665e-01 & -0.052 \\
\bottomrule
\end{tabular}
\caption{Paired multi-dimensional external-baseline comparisons from the perspective of ZIVARI-TLBO. Positive RBC favors ZIVARI-TLBO; Holm correction is applied across comparators within the all-dimensions scope.}
\label{tab:external-multidim-pairwise}
\end{table}

WOA obtains the best overall average rank, 2.914, followed by ZIVARI-TLBO at 3.382 and HHO at 3.465. ZIVARI-TLBO significantly outperforms TLBO, MCTLBO, DE, PSO, and GWO after Holm adjustment. The comparison favors WOA, and the difference from HHO is not significant after Holm adjustment. Figure~\ref{fig:external-multidim-ranks} makes the dimension-dependent pattern visible: ZIVARI-TLBO is comparatively stable, while WOA and HHO improve their rank at higher dimensions in this suite.

\begin{figure}[ht]
\centering
\includegraphics[width=0.95\linewidth]{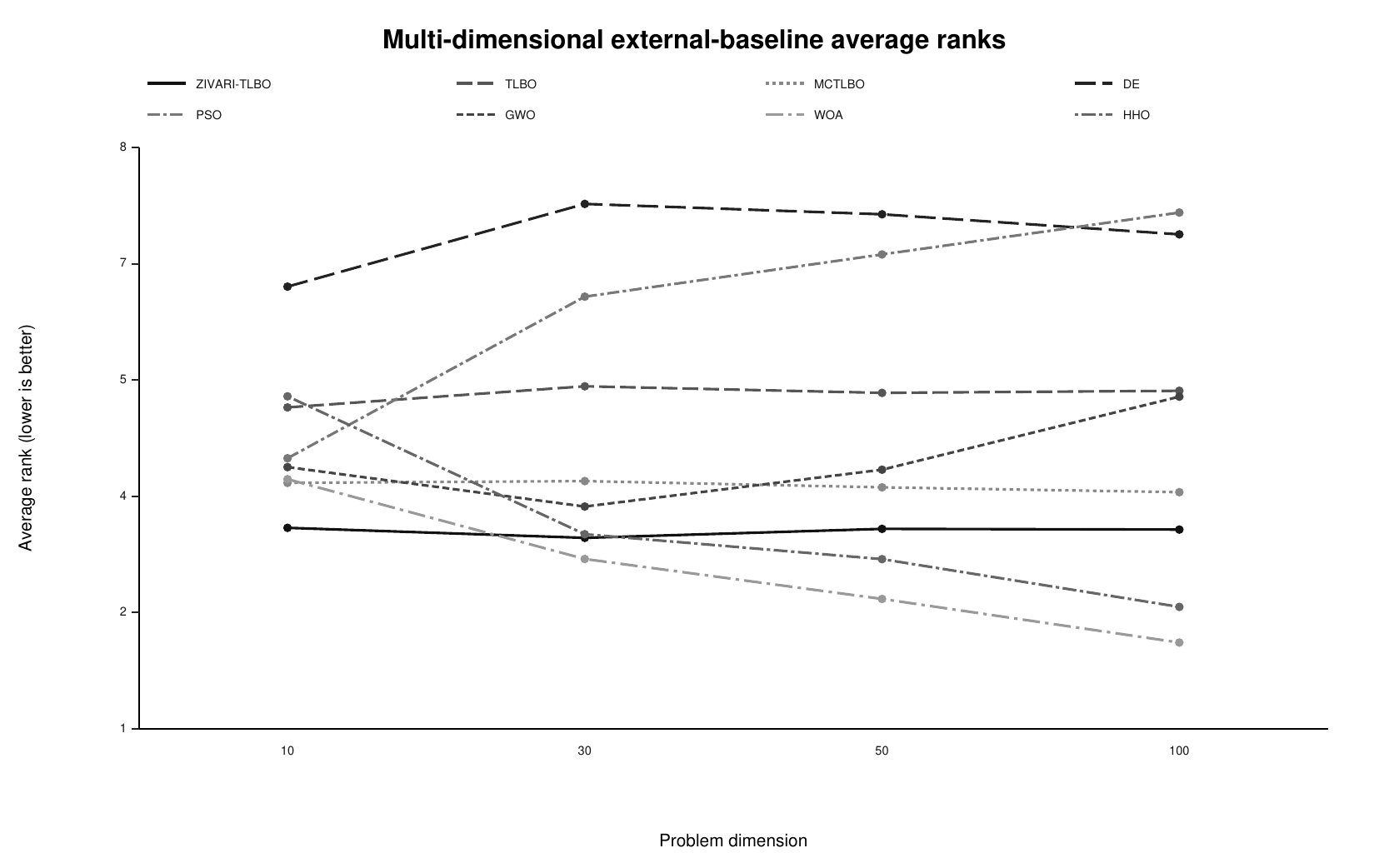}
\caption{Average ranks in the multidimensional external-baseline evaluation. Lower is better. Lines connect separate dimension-specific rankings and do not imply continuous dimensional behavior.}
\label{fig:external-multidim-ranks}
\end{figure}

These results strengthen external comparative coverage but do not establish comprehensive modern state-of-the-art performance. Verified CMA-ES, SHADE, L-SHADE, jSO, and CEC evaluations remain absent.

Figures~\ref{fig:convergence-rastrigin} and \ref{fig:rastrigin-boxplot} illustrate one audited function rather than the entire benchmark suite. Exact zeros on the logarithmic ordinate are displayed one decade below the smallest positive plotted value.

\begin{figure}[ht]
\centering
\includegraphics[width=0.95\linewidth]{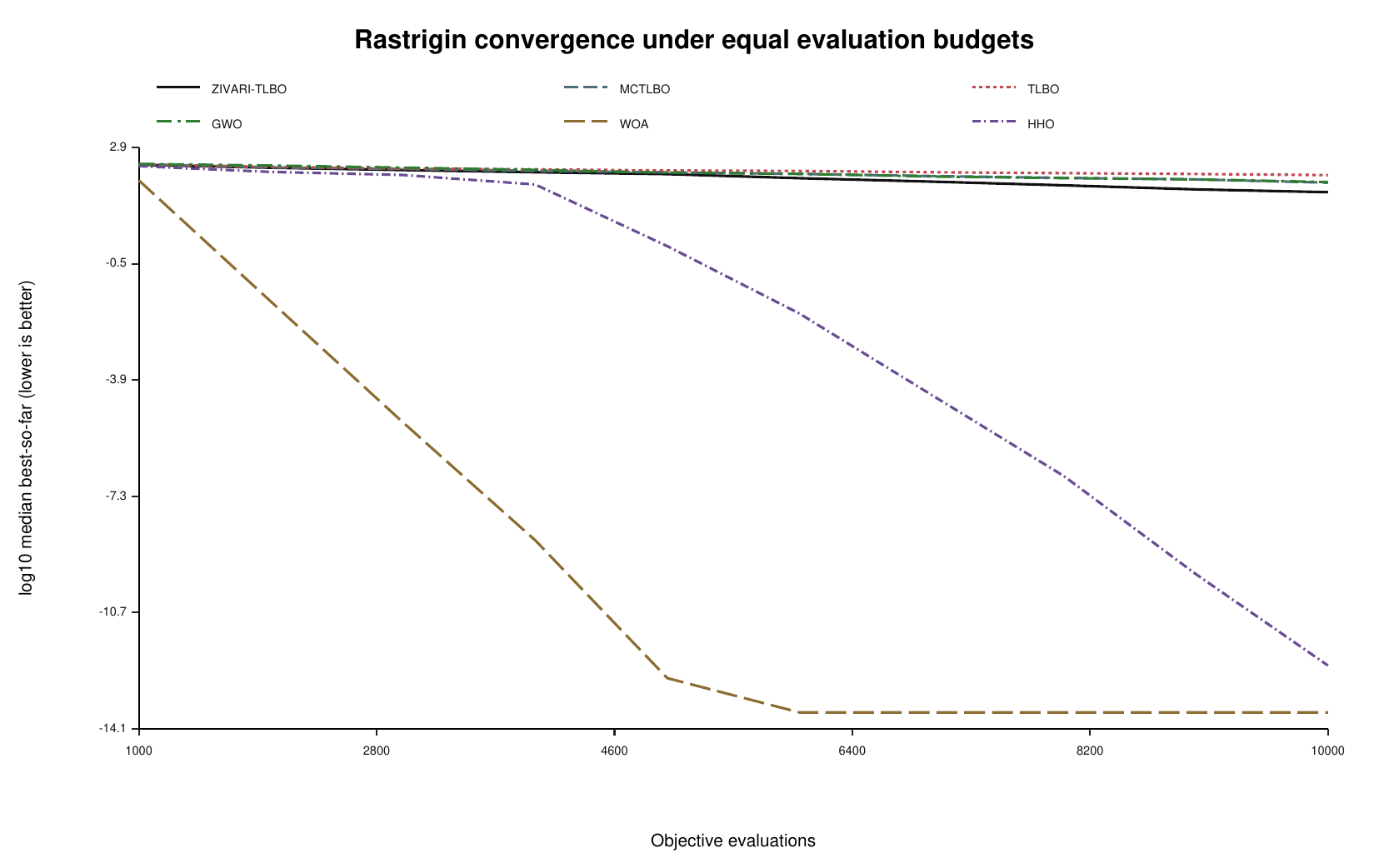}
\caption{Median best-so-far convergence on Rastrigin at dimension 30 under equal evaluation budgets. The ordinate is logarithmic and lower is better; exact zeros are shown one decade below the smallest positive plotted value.}
\label{fig:convergence-rastrigin}
\end{figure}

\begin{figure}[ht]
\centering
\includegraphics[width=0.95\linewidth]{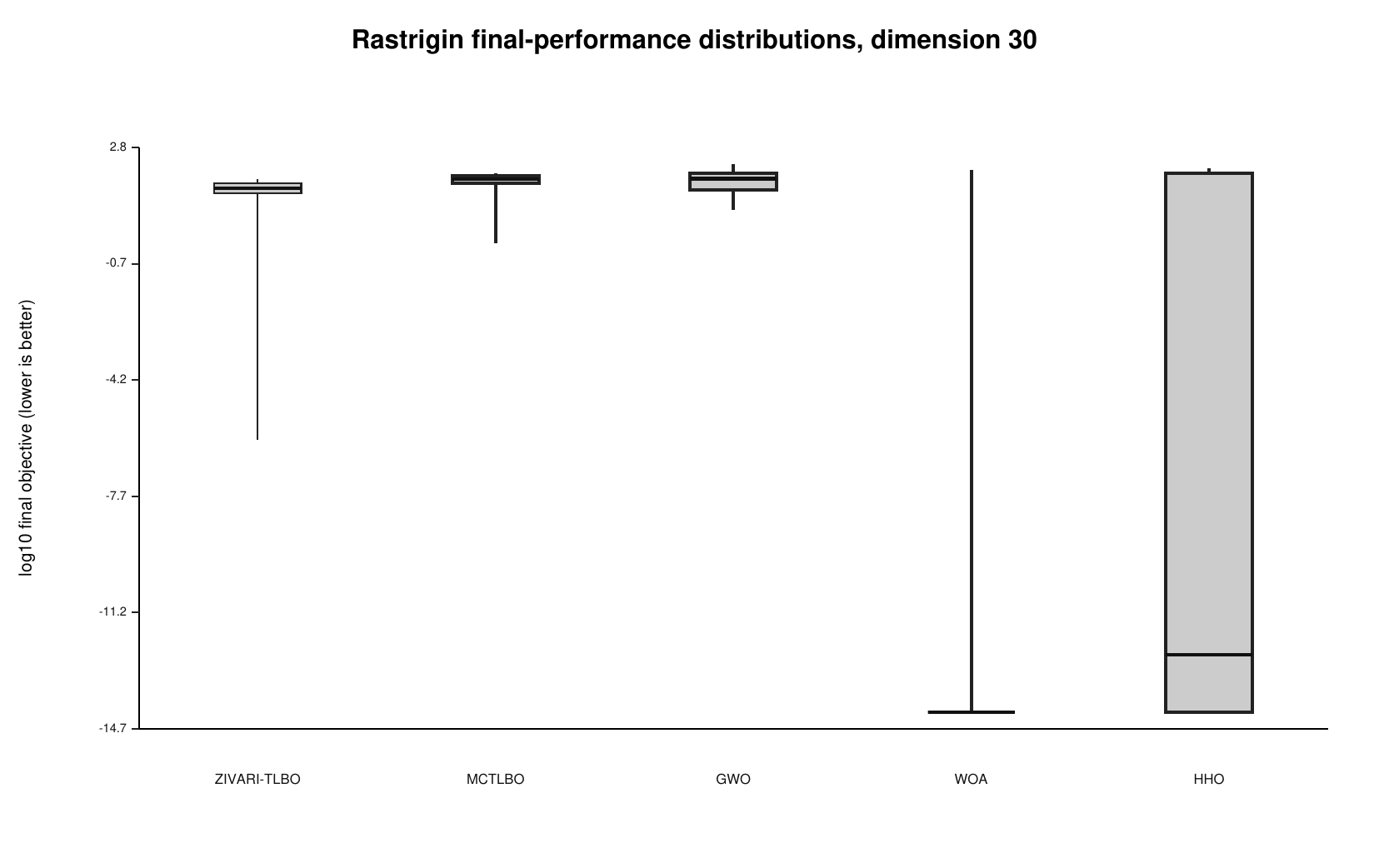}
\caption{Final-performance distributions on Rastrigin at dimension 30 across 30 runs and 10,000 evaluations. The ordinate is logarithmic and lower is better; exact zeros are shown one decade below the smallest positive plotted value.}
\label{fig:rastrigin-boxplot}
\end{figure}

\section{Engineering Design Results}
Table~\ref{tab:zivari-engineering-ranks} reports the overall penalized engineering ranks, while Appendix Tables~\ref{tab:zivari-engineering-summary} and \ref{tab:engineering-feasibility-aware} provide problem-level and feasibility-aware results. ZIVARI-TLBO obtains the second-best overall penalized rank in the current validation, behind DE. The paired comparison with MCTLBO is mixed and not statistically significant after Holm adjustment, whereas the comparison with DE favors DE.

\begin{table}[ht]
\centering
\small
\begin{tabular}{p{0.38\linewidth}rrr}
\toprule
Method & Avg rank seed & Avg rank median & Median rank seed \\
\midrule
DE & 2.25 & 1.8 & 1 \\
ZIVARI-TLBO & 3.137 & 2.1 & 2.75 \\
MCTLBO & 3.357 & 2.7 & 3 \\
GTS-v2 no-focus & 5.263 & 5.6 & 6 \\
TLBO & 5.377 & 6 & 5 \\
GTS grouped-only & 5.533 & 5.2 & 5.5 \\
GTS-v3.1 landscape-aware & 5.7 & 6 & 6 \\
PSO & 7.043 & 7.2 & 8 \\
GWO & 9.253 & 9.4 & 9 \\
HHO & 9.507 & 9.8 & 10 \\
WOA & 9.58 & 10.2 & 11 \\
\bottomrule
\end{tabular}
\caption{Overall engineering design rank statistics across five constrained problems. Lower ranks are better.}
\label{tab:zivari-engineering-ranks}
\end{table}

Figure~\ref{fig:engineering-feasibility} shows that feasibility behavior depends strongly on the problem and constraint formulation. In particular, low feasibility rates on several problems require caution when interpreting small objective differences.

\begin{figure}[ht]
\centering
\includegraphics[width=0.95\linewidth]{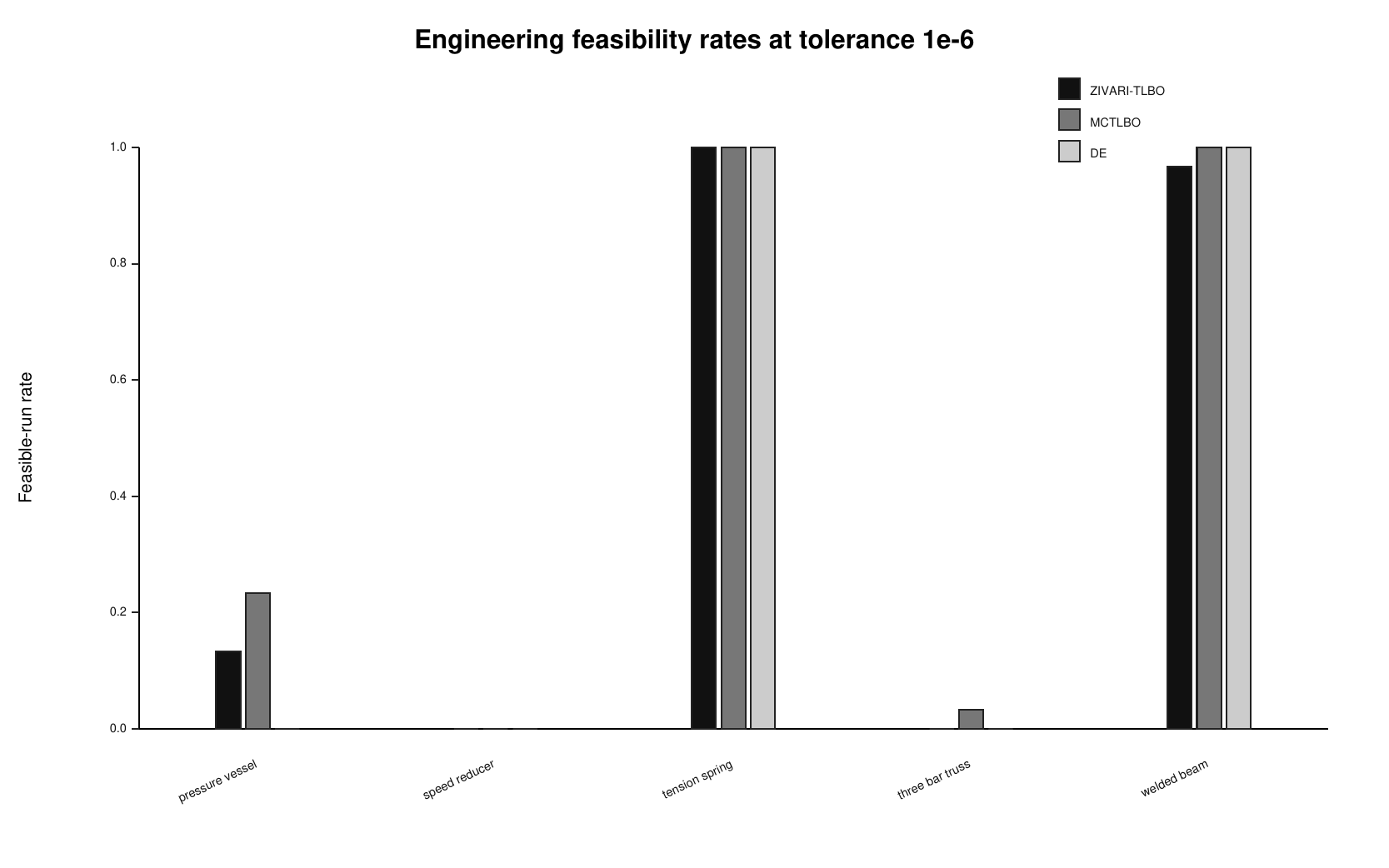}
\caption{Feasible-run rates at tolerance $10^{-6}$ for ZIVARI-TLBO, MCTLBO, and DE on the five audited engineering problems under the current static-penalty formulation.}
\label{fig:engineering-feasibility}
\end{figure}

\subsection{Feasibility-Aware Interpretation of Engineering Results}
Under the strict $10^{-6}$ tolerance, ZIVARI-TLBO obtains feasible solutions in 29/30 welded-beam runs and 30/30 tension-spring runs, but only 4/30 pressure-vessel runs and no feasible speed-reducer or three-bar-truss runs. Consequently, low raw objective values from infeasible runs must not be interpreted as valid engineering improvements. Under the external Deb-style feasibility-first analysis, DE obtains the best average rank (4.933), ZIVARI-TLBO ranks second (5.200), and MCTLBO ranks third (5.293). This ranking is descriptive of the audited outputs and does not replace the optimizers' static-penalty search behavior.

The engineering evidence is therefore mixed and constraint-handling sensitive. It supports reporting competitive behavior on some problems, but not general engineering superiority. Stronger engineering validation would require feasibility-preserving or adaptive constraint handling, independently verified formulations, and substantially better feasible-run rates on the currently difficult problems \cite{deb2000constraints,coello2002constraints}.

\section{Statistical Analysis}
Pairwise comparisons use matched function-seed observations for the classical layers and matched problem-seed observations for engineering. The Q1-upgrade analysis reports two-sided Wilcoxon signed-rank tests using a tie-corrected asymptotic normal approximation without continuity correction, wins/ties/losses, Holm-adjusted $p$-values, and rank-biserial effect sizes \cite{derrac2011stats}. Zero differences are excluded from the nonzero Wilcoxon count and remain visible as ties. Previously archived statistical outputs whose execution mode cannot be recovered remain explicitly marked as unspecified rather than silently reinterpreted.

For the eight-method multidimensional comparison, a Friedman test over 960 complete function-seed blocks rejects equal rank distributions ($\chi^2_F=2256.787$, $df=7$, $p<10^{-300}$ under the chi-square approximation). Pairwise Wilcoxon tests with Holm correction are reported separately. The direct relay ablation has two conditions, so a Friedman test is not applied to that comparison. Statistical significance is interpreted with ranks, W/T/L, effect sizes, and benchmark coverage, following established recommendations for multiple algorithm comparisons \cite{demsar2006stats,derrac2011stats}. Figure~\ref{fig:average-ranks} retains separate rank summaries for the earlier TLBO-family and engineering validation layers; they are not pooled into a global claim.

\begin{figure}[ht]
\centering
\includegraphics[width=0.95\linewidth]{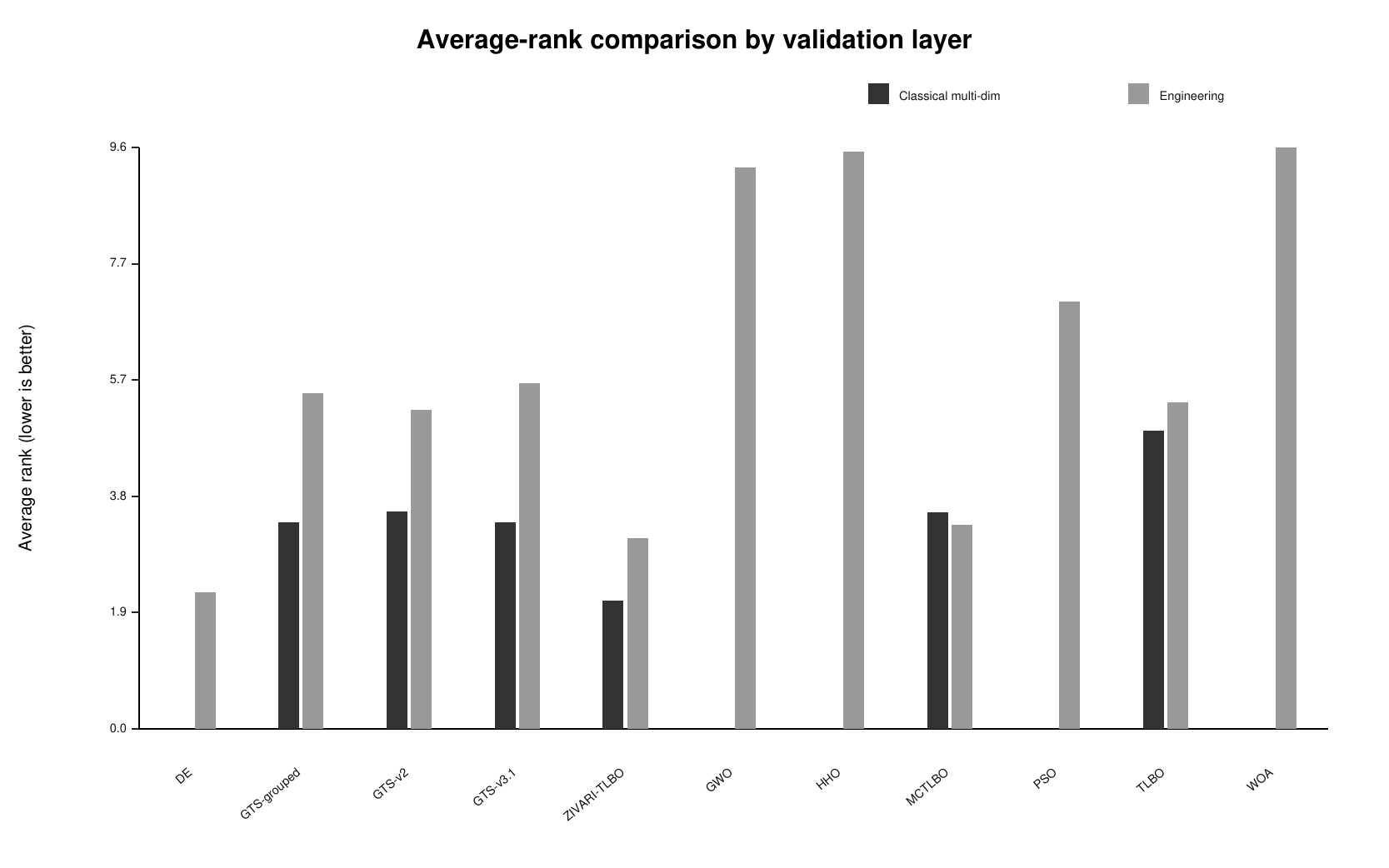}
\caption{Average-rank summaries for the multidimensional classical TLBO-family layer and the engineering layer. The algorithm sets differ, so the two series are not a single pooled ranking.}
\label{fig:average-ranks}
\end{figure}

\section{Discussion}
\subsection{Interpretation of the Relay}
The audited results suggest that evaluated-elite relay can improve information sharing without spending additional objective-function calls on the relay itself. Under the reported evaluation budgets, this is meaningful because a soft or noisy transfer would have to consume part of the same finite budget. The fixed ring topology also avoids benchmark-name rules and avoids ranking source--target pairs adaptively at relay time.

The mechanism should not be interpreted as creating a better solution at the instant of transfer. An exact copy cannot improve the global best because the copied elite was already evaluated. Its plausible benefit is indirect: the receiving group begins subsequent local updates from a stronger internal reference. The matched relay-disabled ablation provides strong evidence that this interaction is beneficial across the adopted classical protocol, but it does not constitute a general causal proof outside that protocol.

\subsection{Strengths and Mixed Evidence}
The direct ablation is the strongest evidence for the evaluated-elite relay contribution because it controls the inherited grouped controller and evaluation protocol. The broader eight-method results are more heterogeneous: ZIVARI-TLBO ranks second overall, WOA ranks first, and HHO is statistically indistinguishable from ZIVARI-TLBO after Holm adjustment. Engineering evidence is also mixed and limited by low feasibility on several problems. These negative and mixed findings delimit the supported claims and reveal where stronger validation is needed.

\subsection{Development Diagnostics}
Earlier GTS variants are retained as diagnostic evidence, not as additional proposed methods. GTS-v3.1 established the inherited benchmark-agnostic landscape-aware controller; later experiments indicated that simple activation and short-horizon intervention success were not sufficient selection criteria. The final paper therefore freezes one interpretable mechanism and shifts emphasis from further versioning to validation, direct ablation, diagnostics, and claim safety.

\section{Future Work}
\subsection{Toward Future V5 Extensions}
The present study freezes ZIVARI-TLBO as \texttt{gts-v4-cm-fixed}; no V5 method is implemented or validated. A separate future line may investigate adaptive relay scheduling in which relay frequency, sender selection, or receiver selection responds to diversity, stagnation, improvement rate, or persistent regime diagnostics. Perturbative, soft, or noisy relay would be a distinct method because modified vectors generally require additional objective-function evaluations.

The more immediate priority is stronger independent validation of the frozen method: verified CEC integration with traceable definitions and data \cite{cec2022}, modern external baselines such as CMA-ES and success-history DE variants, alternative constraint-handling strategies, hardware-controlled runtime experiments, and real-world black-box cases. Any future extension would require a separate implementation, direct ablation, equal-budget benchmarking, and independent statistical validation. Broader robustness is a research objective, not an established result.

\section{Limitations}
The evidence does not support a universal superiority claim. Although the external comparison now covers dimensions 10, 30, 50, and 100, it omits verified CMA-ES, SHADE, L-SHADE, jSO, and other modern comparators. WOA obtains a better overall rank than ZIVARI-TLBO, and HHO is not significantly different after Holm adjustment. The direct ablation isolates the relay within the tested grouped controller, but it does not establish global convergence or universal causal benefit. The engineering analysis uses one static-penalty formulation, and several problem--method combinations have low or zero strict-feasibility rates. Runtime hardware details are unspecified. No verified CEC implementation, official data package, or validated wrapper exists in the repository; consequently, no CEC result or CEC superiority claim is reported.

\section{Reproducibility Statement}
The paper-facing method name is ZIVARI-TLBO and the frozen code alias is \texttt{gts-v4-cm-fixed}. Classical raw outputs cover eight functions, dimensions 10/30/50/100, eight focal external methods, 30 independent matched seeds, and equal 10,000-evaluation budgets. Direct-ablation outputs preserve the controlled comparison with the same landscape-aware grouped controller without relay. Engineering raw outputs cover five constrained problems with the same run count and evaluation budget, including objective, penalized objective, feasibility, violation, and design fields. The full reproducibility archive contains source data, a standardized long-form dataset, statistical outputs, figure and table manifests, direct-ablation protocol, run log, claim limitations, and scripts to regenerate analyses, figures, tables, and archive audits. Fields that cannot be recovered from archived evidence, including timing hardware and the execution mode of older archived Wilcoxon results, are marked as unspecified. The frozen \texttt{TI.py} SHA256 recorded before and after the new external-baseline runs is \texttt{0dd283a3084a7c07624155c99537ed11173b4908ecf92bce44fd596919b5d9d2}. Adaptive, noisy, soft, perturbative, and V5 relay variants are not part of the current validated method. CEC remains future work because no verified implementation is available.

\section{Conclusion}
ZIVARI-TLBO adds a fixed ring evaluated-elite relay to a grouped TLBO framework with an inherited population-state controller. The exact relay transfers stored solutions and objective values without additional objective-function calls, while all generated candidates remain subject to the common evaluation budget. A matched multidimensional ablation strongly supports the relay contribution within the adopted protocol. In the broader eight-method comparison, ZIVARI-TLBO ranks second rather than first, and feasibility-aware engineering findings remain mixed. The appropriate next step is independent review, verified CEC and modern-baseline validation, stronger constraint handling, and journal-template adaptation rather than another optimizer version.

\section*{Declaration of Generative AI and AI-Assisted Technologies}
During the preparation of this manuscript, the author used AI-assisted tools, including ChatGPT and Codex, to support language editing, structural refinement, LaTeX organization, reproducibility-package preparation, and claim-safety review. The proposed method, experimental interpretation, mathematical formulation, results, conclusions, and all final scientific claims were reviewed, verified, and approved by the author. No AI tool is listed as an author, and the author takes full responsibility for the accuracy, integrity, originality, and validity of the manuscript.

\appendix
\section{Detailed Validation Tables}
Appendix Table~\ref{tab:zivari-classical-dim-ranks} reports the original TLBO-family dimension-wise ranks. Appendix Tables~\ref{tab:zivari-engineering-summary} and \ref{tab:engineering-feasibility-aware} report detailed engineering and feasibility-aware summaries.

\small
\begin{longtable}{p{0.08\linewidth}p{0.43\linewidth}rr}
\caption{Dimension-wise classical benchmark rank statistics for the frozen ZIVARI-TLBO validation layer.}\label{tab:zivari-classical-dim-ranks}\\
\toprule
Dim & Method & Avg rank seed & Avg rank median \\
\midrule
\endfirsthead
\toprule
Dim & Method & Avg rank seed & Avg rank median \\
\midrule
\endhead
10 & ZIVARI-TLBO & 2.104 & 1.625 \\
10 & MCTLBO & 3.404 & 3.5 \\
10 & GTS grouped-only & 3.438 & 3.438 \\
10 & GTS-v3.1 landscape-aware & 3.471 & 3.812 \\
10 & GTS-v2 no-focus & 3.871 & 3.875 \\
10 & TLBO & 4.713 & 4.75 \\
30 & ZIVARI-TLBO & 1.933 & 1.625 \\
30 & GTS-v2 no-focus & 3.458 & 3.25 \\
30 & MCTLBO & 3.483 & 3.375 \\
30 & GTS-v3.1 landscape-aware & 3.513 & 3.375 \\
30 & GTS grouped-only & 3.6 & 4.25 \\
30 & TLBO & 5.013 & 5.125 \\
50 & ZIVARI-TLBO & 2.208 & 2 \\
50 & GTS grouped-only & 3.269 & 3.5 \\
50 & GTS-v3.1 landscape-aware & 3.373 & 3.625 \\
50 & GTS-v2 no-focus & 3.483 & 3.375 \\
50 & MCTLBO & 3.629 & 3.375 \\
50 & TLBO & 5.037 & 5.125 \\
100 & ZIVARI-TLBO & 2.225 & 2.375 \\
100 & GTS-v3.1 landscape-aware & 3.246 & 3 \\
100 & GTS grouped-only & 3.312 & 3.75 \\
100 & GTS-v2 no-focus & 3.525 & 3.375 \\
100 & MCTLBO & 3.779 & 3.625 \\
100 & TLBO & 4.912 & 4.875 \\
\bottomrule
\end{longtable}

\small
\begin{longtable}{p{0.16\linewidth}p{0.24\linewidth}rrrr}
\caption{Engineering design validation summary. Constraint feasibility is reported using the 1e-6 tolerance column; lower objective and rank are better.}\label{tab:zivari-engineering-summary}\\
\toprule
Problem & Method & Best obj. & Median obj. & Feas. 1e-6 & Med. rank \\
\midrule
\endfirsthead
\toprule
Problem & Method & Best obj. & Median obj. & Feas. 1e-6 & Med. rank \\
\midrule
\endhead
pressure vessel & DE & 5.882e+03 & 5.882e+03 & 0 & 1 \\
pressure vessel & ZIVARI-TLBO & 5.882e+03 & 5.883e+03 & 0.133 & 2 \\
pressure vessel & MCTLBO & 5.881e+03 & 5.893e+03 & 0.233 & 3 \\
pressure vessel & GTS-v3.1 landscape-aware & 5.882e+03 & 5.898e+03 & 0.2 & 4 \\
pressure vessel & GTS grouped-only & 5.882e+03 & 5.920e+03 & 0.367 & 5 \\
pressure vessel & TLBO & 5.880e+03 & 5.964e+03 & 0.5 & 6 \\
pressure vessel & GTS-v2 no-focus & 5.882e+03 & 5.998e+03 & 0.333 & 7 \\
pressure vessel & PSO & 5.882e+03 & 6.196e+03 & 0 & 8 \\
pressure vessel & GWO & 6.031e+03 & 6.427e+03 & 0.8 & 9 \\
pressure vessel & HHO & 6.253e+03 & 6.806e+03 & 0.367 & 10 \\
pressure vessel & WOA & 6.446e+03 & 4.622e+04 & 0.233 & 11 \\
speed reducer & ZIVARI-TLBO & 2.996e+03 & 2.996e+03 & 0 & 1.5 \\
speed reducer & MCTLBO & 2.996e+03 & 2.996e+03 & 0 & 1.5 \\
speed reducer & TLBO & 2.996e+03 & 2.996e+03 & 0 & 3 \\
speed reducer & DE & 2.996e+03 & 2.996e+03 & 0 & 4 \\
speed reducer & GTS-v2 no-focus & 2.996e+03 & 2.996e+03 & 0 & 5 \\
speed reducer & GTS grouped-only & 2.996e+03 & 2.996e+03 & 0 & 6 \\
speed reducer & GTS-v3.1 landscape-aware & 2.996e+03 & 2.996e+03 & 0 & 7 \\
speed reducer & WOA & 2.996e+03 & 2.996e+03 & 0 & 8 \\
speed reducer & HHO & 2.996e+03 & 3.000e+03 & 0 & 9 \\
speed reducer & GWO & 3.005e+03 & 3.014e+03 & 0.7 & 10 \\
speed reducer & PSO & 2.996e+03 & 3.036e+03 & 0 & 11 \\
tension spring & DE & 0.013 & 0.013 & 1 & 1 \\
tension spring & GTS-v2 no-focus & 0.013 & 0.013 & 1 & 2 \\
tension spring & MCTLBO & 0.013 & 0.013 & 1 & 3 \\
tension spring & ZIVARI-TLBO & 0.013 & 0.013 & 1 & 4 \\
tension spring & GTS-v3.1 landscape-aware & 0.013 & 0.013 & 1 & 5 \\
tension spring & GTS grouped-only & 0.013 & 0.013 & 1 & 6 \\
tension spring & TLBO & 0.013 & 0.013 & 1 & 7 \\
tension spring & PSO & 0.013 & 0.013 & 1 & 8 \\
tension spring & GWO & 0.013 & 0.013 & 1 & 9 \\
tension spring & WOA & 0.013 & 0.013 & 1 & 10 \\
tension spring & HHO & 0.013 & 0.013 & 1 & 11 \\
three bar truss & DE & 263.895 & 263.895 & 0 & 1 \\
three bar truss & ZIVARI-TLBO & 263.895 & 263.895 & 0 & 2 \\
three bar truss & MCTLBO & 263.895 & 263.895 & 0.033 & 3 \\
three bar truss & PSO & 263.895 & 263.895 & 0 & 4 \\
three bar truss & GTS grouped-only & 263.893 & 263.895 & 0.067 & 5 \\
three bar truss & GTS-v2 no-focus & 263.894 & 263.895 & 0.133 & 6 \\
three bar truss & GTS-v3.1 landscape-aware & 263.893 & 263.895 & 0.1 & 7 \\
three bar truss & TLBO & 263.893 & 263.895 & 0.1 & 8 \\
three bar truss & HHO & 263.896 & 263.972 & 0.067 & 9 \\
three bar truss & GWO & 263.887 & 264.024 & 0.867 & 10 \\
three bar truss & WOA & 263.899 & 265.094 & 0 & 11 \\
welded beam & ZIVARI-TLBO & 1.725 & 1.725 & 0.967 & 1 \\
welded beam & DE & 1.725 & 1.725 & 1 & 2 \\
welded beam & MCTLBO & 1.725 & 1.725 & 1 & 3 \\
welded beam & GTS grouped-only & 1.725 & 1.725 & 0.967 & 4 \\
welded beam & PSO & 1.725 & 1.725 & 0.933 & 5 \\
welded beam & TLBO & 1.725 & 1.725 & 0.967 & 6 \\
welded beam & GTS-v3.1 landscape-aware & 1.725 & 1.725 & 1 & 7 \\
welded beam & GTS-v2 no-focus & 1.725 & 1.725 & 1 & 8 \\
welded beam & GWO & 1.744 & 1.8 & 1 & 9 \\
welded beam & HHO & 1.739 & 2.046 & 1 & 10 \\
welded beam & WOA & 1.793 & 3.067 & 1 & 11 \\
\bottomrule
\end{longtable}

\begin{longtable}{p{0.20\linewidth}p{0.20\linewidth}rrrr}
\caption{Feasibility-aware engineering summary. Feasible objective statistics exclude infeasible runs; NA indicates that no run met the strict $10^{-6}$ feasibility tolerance.}\label{tab:engineering-feasibility-aware}\\
\toprule
Problem & Method & Feasible & Rate & Feasible best & Mean violation \\
\midrule
\endfirsthead
\toprule
Problem & Method & Feasible & Rate & Feasible best & Mean violation \\
\midrule
\endhead
pressure vessel & DE & 0/30 & 0.000 & NA & 5.068e-04 \\
pressure vessel & GTS grouped-only & 11/30 & 0.367 & 5903 & 2.970e-04 \\
pressure vessel & GTS-v2 no-focus & 10/30 & 0.333 & 5894 & 3.207e-04 \\
pressure vessel & GTS-v3.1 & 6/30 & 0.200 & 5900 & 4.250e-04 \\
pressure vessel & ZIVARI-TLBO & 4/30 & 0.133 & 6014 & 3.828e-04 \\
pressure vessel & GWO & 24/30 & 0.800 & 6031 & 2.563e-04 \\
pressure vessel & HHO & 11/30 & 0.367 & 6560 & 2.866e-04 \\
pressure vessel & MCTLBO & 7/30 & 0.233 & 5898 & 4.816e-04 \\
pressure vessel & PSO & 0/30 & 0.000 & NA & 4.730e-04 \\
pressure vessel & TLBO & 15/30 & 0.500 & 5889 & 4.055e-04 \\
pressure vessel & WOA & 7/30 & 0.233 & 6446 & 5.317e-04 \\
speed reducer & DE & 0/30 & 0.000 & NA & 1.390e-04 \\
speed reducer & GTS grouped-only & 0/30 & 0.000 & NA & 1.385e-04 \\
speed reducer & GTS-v2 no-focus & 0/30 & 0.000 & NA & 1.386e-04 \\
speed reducer & GTS-v3.1 & 0/30 & 0.000 & NA & 1.378e-04 \\
speed reducer & ZIVARI-TLBO & 0/30 & 0.000 & NA & 1.391e-04 \\
speed reducer & GWO & 21/30 & 0.700 & 3005 & 1.492e-04 \\
speed reducer & HHO & 0/30 & 0.000 & NA & 1.960e-04 \\
speed reducer & MCTLBO & 0/30 & 0.000 & NA & 1.390e-04 \\
speed reducer & PSO & 0/30 & 0.000 & NA & 9.144e-05 \\
speed reducer & TLBO & 0/30 & 0.000 & NA & 1.390e-04 \\
speed reducer & WOA & 0/30 & 0.000 & NA & 1.665e-04 \\
tension spring & DE & 30/30 & 1.000 & 0.01267 & 1.155e-09 \\
tension spring & GTS grouped-only & 30/30 & 1.000 & 0.01267 & 0 \\
tension spring & GTS-v2 no-focus & 30/30 & 1.000 & 0.01267 & 1.060e-08 \\
tension spring & GTS-v3.1 & 30/30 & 1.000 & 0.01267 & 1.508e-10 \\
tension spring & ZIVARI-TLBO & 30/30 & 1.000 & 0.01267 & 3.599e-10 \\
tension spring & GWO & 30/30 & 1.000 & 0.01269 & 0 \\
tension spring & HHO & 30/30 & 1.000 & 0.01267 & 7.334e-09 \\
tension spring & MCTLBO & 30/30 & 1.000 & 0.01267 & 2.282e-10 \\
tension spring & PSO & 30/30 & 1.000 & 0.01267 & 1.083e-09 \\
tension spring & TLBO & 30/30 & 1.000 & 0.01267 & 0 \\
tension spring & WOA & 30/30 & 1.000 & 0.01267 & 1.761e-09 \\
three bar truss & DE & 0/30 & 0.000 & NA & 6.597e-06 \\
three bar truss & GTS grouped-only & 2/30 & 0.067 & 263.9 & 7.635e-06 \\
three bar truss & GTS-v2 no-focus & 4/30 & 0.133 & 263.9 & 5.211e-06 \\
three bar truss & GTS-v3.1 & 3/30 & 0.100 & 263.9 & 6.228e-06 \\
three bar truss & ZIVARI-TLBO & 0/30 & 0.000 & NA & 6.402e-06 \\
three bar truss & GWO & 26/30 & 0.867 & 263.9 & 5.977e-06 \\
three bar truss & HHO & 2/30 & 0.067 & 264 & 5.978e-06 \\
three bar truss & MCTLBO & 1/30 & 0.033 & 263.9 & 6.593e-06 \\
three bar truss & PSO & 0/30 & 0.000 & NA & 6.596e-06 \\
three bar truss & TLBO & 3/30 & 0.100 & 263.9 & 6.519e-06 \\
three bar truss & WOA & 0/30 & 0.000 & NA & 6.865e-06 \\
welded beam & DE & 30/30 & 1.000 & 1.725 & 7.132e-09 \\
welded beam & GTS grouped-only & 29/30 & 0.967 & 1.725 & 9.347e-08 \\
welded beam & GTS-v2 no-focus & 30/30 & 1.000 & 1.725 & 0 \\
welded beam & GTS-v3.1 & 30/30 & 1.000 & 1.725 & 1.735e-08 \\
welded beam & ZIVARI-TLBO & 29/30 & 0.967 & 1.725 & 7.155e-08 \\
welded beam & GWO & 30/30 & 1.000 & 1.744 & 0 \\
welded beam & HHO & 30/30 & 1.000 & 1.739 & 0 \\
welded beam & MCTLBO & 30/30 & 1.000 & 1.725 & 4.458e-08 \\
welded beam & PSO & 28/30 & 0.933 & 1.725 & 1.013e-07 \\
welded beam & TLBO & 29/30 & 0.967 & 1.725 & 3.640e-08 \\
welded beam & WOA & 30/30 & 1.000 & 1.793 & 6.943e-09 \\
\bottomrule
\end{longtable}

\bibliographystyle{plain}
\bibliography{references}

\end{document}